\documentclass[11pt]{article}

\usepackage[margin=1in]{geometry}
\usepackage{amsthm, amsmath, amssymb,graphicx,geometry,mathrsfs}
\usepackage{dsfont,wrapfig,caption,changepage,float}
\usepackage[english]{babel}
\usepackage[utf8]{inputenc}
\usepackage[colorinlistoftodos]{todonotes}

\newcommand{\comment}[1]{}

\newtheorem{lemma}{Lemma}
\newtheorem{thm}{Theorem}

\newtheorem{prop}{Proposition}
\newtheorem{defn}{Definition}
\newtheorem{exmp}{Example}

\DeclareMathOperator*{\argmax}{arg\,max}

\title{Convex Combination Belief Propagation}
\author{Anna Grim and Pedro Felzenszwalb}

\begin{document}

\maketitle

\begin{abstract}
We present new message passing algorithms for performing inference with graphical models.  Our methods are designed for the most difficult inference problems where loopy belief propagation and other heuristics fail to converge.  Belief propagation is guaranteed to converge when the underlying graphical model is acyclic, but can fail to converge and is sensitive to initialization when the underlying graph has complex topology. This paper describes modifications to the standard belief propagation algorithms that lead to methods that converge to unique solutions on graphical models with arbitrary topology and potential functions.  
\end{abstract}

\section{Introduction}

Graphical models provide a natural framework for probabilistic modeling and statistical inference with a large number of random variables. In this setting, statistical inference involves computing either marginal distributions for the individual random variables, or joint configurations with maximum probability. Exact inference is NP-hard in both cases, so it is essential to develop approximation algorithms that make inference computationally tractable \cite{cooper90}. 

Belief propagation (BP) is a widely used message passing algorithm that can be used to perform either exact or approximate inference in graphical models. This algorithm was introduced by Judea Pearl in the early 1980s as a method for performing exact inference on acyclic graphs in polynomial time \cite{pearl83}.  For graphical models with cycles BP is a heuristic (often called \emph{loopy} BP) that can be used to perform approximate inference. This method has gained popularity within the artificial intelligence community since it obtains state of the art results in certain settings such as error correcting codes and image analysis (see, e.g., \cite{chen02}, \cite{fossorier99}, \cite{frey97}, \cite{murphy99}, \cite{sun03}). However, BP is also known for failing to converge and for being sensitive to initialization.
These issues are the main motivations behind this paper as well as many other works over the last few decades. 

The BP literature has primarily focused on two key problems: (1) understanding why loopy belief propagation fails to converge and (2) using this knowledge to develop alternative local message passing algorithms. Broadly speaking, whether the algorithm converges depends on the topology of the graph and the nature of the potential functions in the graphical model.

We present \emph{convex combination belief propagation} (CCBP), a convergent alternative to traditional belief propagation inspired by the work in \cite{felzenszwalb19}.  Our primary motivation for developing CCBP is that loopy belief propagation is known to fail to converge on graphs with complex topology. CCBP addresses this problem by weighting the incoming messages when computing an outgoing message.  This mitigates the effect of feedback loops and leads to globally convergent versions of both the sum-product and max-product BP algorithms.  

Felzenszwalb and Svaiter (2019) used a type of non-linear diffusion to obtain globally convergent methods for approximate inference in graphical models \cite{felzenszwalb19}.  The approach we take in this paper is closely related.  
CCBP involves fixed point iteration with contractive maps.  For a fixed set of valid weights the process converges to a unique solution regardless of the topology of the graph 
and message initialization. It is simple to determine weights that guarantee uniqueness of the fixed point and convergence of fixed point iteration.  The weights depend only on the local topology of the graph and do not
depend on the strength or nature of the potentials in the graphical model.  Moreover our weighting scheme can be readily incorporated into existing implementations of BP.  

Tatikonda and Jordan (2002) showed that BP converges when the model satisfies Dobrushin's condition (see \cite{georgii11}, \cite{tatikonda03}, \cite{tatikonda02}). Intuitively, this condition states that the coupling factor (induced by the potentials) between variables must be sufficiently small. Ihler et al.(2005) derived a similar condition by bounding message error in terms of the dynamic range of the potentials \cite{ihler05}. Heskes (2004) derived a sufficient condition for the uniqueness of BP fixed points by equivalently determining when the Bethe free energy is convex (see \cite{heskes02}, \cite{heskes04}, \cite{yedidia05}, \cite{watanabe09}). The resulting conditions state that both cycles and the strength of the potentials affect whether this free energy is convex. Martin and Lasgouttes (2012) derive a sufficient condition for local convergence in terms of the graph structure and the beliefs values at the fixed point \cite{martin12}. Their result provides insight into why BP is more likely to converge on sparser graphs. 

These works have influenced the development of alternative local message passing algorithms that improve convergence by damping the potentials and/or mitigating the effects of feedback loops caused by message passing on graphs with cycles. Yedida et al. (2000) developed \emph{generalized} belief propagation in which the number of cycles is reduced by formulating message passing between regions of nodes \cite{yedidia00}. One drawback is that the performance is highly dependent upon how the graph is partitioned into regions, which can be a challenging task. Wainwright et al. (2003) developed \emph{tree-reweighted} belief propagation by maximizing a lower bound on the log partition function via convex combinations of tree-structured distributions \cite{wainwright03}. Kolmorogov (2006) introduced \emph{sequential} tree-reweighted belief propagation \cite{kolmogorov06} which improves the convergence of Wainwright's method by utilizing sequential (as opposed to parallel) updates. Although these alternatives have better convergence properties, they are more computationally demanding.

Roosta and Wainwright (2008) introduced a reweighted sum-product algorithm that incorporates edge-weights on the potentials and messages in each update \cite{roosta08}. The algorithm converges when the spectral radius of the update operator is bounded by one. However, it may be difficult to determine a set of weights that satisfies this condition. Knoll et al. (2018) introduced a homotopy continuation based approach called \emph{self-guided} belief propagation \cite{knoll18}. Their method interpolates between a pairwise model and a simplification of that model with only unary potentials. Although the method always returns a solution, it's only guaranteed to converge on the first time step.

The remaining of the paper is organized as follows. Section \ref{sec:background} provides a brief overview of graphical models and belief propagation, while also establishing basic notation. We introduce CCBP in Section \ref{sec:ccbp}, then prove several theoretical properties of the algorithm. Section \ref{sec:exps} discusses several numerical experiments that evaluate the performance of our algorithm. 

\section{Background}\label{sec:background}

\subsection{Probabilistic Graphical Models}

We consider pairwise undirected graphical models (Markov random fields).
Let $G=(V,E)$ be an undirected graph with the vertex set $V=\{1,\ldots,n\}$.  We use $N(i)=\{j \,:\, \{i,j\}\in E\}$ to denote the set of neighbors of node $i\in V$. Let $X=(X_1,\ldots,X_n)$ be a random vector where $X_i$ is a random variable with a set of possible outcomes $\Omega=\{1,\ldots,m\}$. A configuration of the random vector is given by $x=(x_1,\ldots,x_n)\in\Omega^n$.  
The probability of a configuration is given by the joint distribution:
\begin{equation*}
    \mathbb{P}(X=x)=\frac{1}{Z}\prod_{ i\in V}\phi_i(x_i)\prod_{\{i,j\}\in E}\psi_{ij}(x_i,x_j),
\end{equation*}
where $Z$ is a normalization constant. The functions $\phi_i$ and $\psi_{ij}$ are referred as \emph{potentials} and assumed to be positive.  The graphical model is said to be \emph{pairwise} since the joint distribution is a product potentials that depend on pairs of random variables.

Statistical inference is a central computational challenge in many applications. Exact inference is generally intractable for arbitrary distributions, especially when the joint distribution is defined over a large number of random variables. There are two inference tasks that frequently appear in applications: (1) finding the most probable (MAP) solution and (2) computing marginals.
\begin{enumerate}
    \item \emph{MAP Inference}. Determine the most probable state of the random vector,
    \begin{equation*}
    \hat x_\text{MAP}=\argmax_{x\in\Omega^n}\,\mathbb{P}(X=x).
    \end{equation*}
    \item \emph{Marginal Inference}. Compute the marginal distribution of each random variable $i\in V$,
    \begin{equation*}
    \mathbb{P}(X_i=\tau)= \sum_{\{x\,:\, x_i=\tau\}}\mathbb{P}(X=x).
    \end{equation*}
\end{enumerate}
There are two closely related belief propagation algorithms that can be used to address the two inference tasks. The max-product algorithm performs MAP inference, while the sum-product algorithm performs marginal inference. 

\subsection{Max-Product Algorithm}

The max-product algorithm finds the MAP solution by computing max-marginals for each node. Given any node $i\in V$, the max-marginal of this node is
\begin{equation*}
    p_i(\tau)=\max_{\{x\,:\, x_i=\tau\}}\mathbb P(X=x).
\end{equation*}
Note that this definition bears a close resemblance to the definition of a marginal distribution.  The value $p_i(\tau)$ specifies the maximum probability of a configuration where the state of the $i$-th random variable is $\tau$. As long as there are no ties, the exact MAP solution can be obtained by selecting
$\hat x_i=\arg\max\,\,p_i(\tau)$ for every node $i\in V$. 

The main idea behind the max-product BP algorithm is to use dynamic programming to break up the calculation of the max-marginals into subproblems. This results in a local message passing algorithm, where computing messages is equivalent to solving subproblems. Each message sent from a node to a neighboring node incorporates messages from the other neighbors, and information propagates throughout the graph as illustrated in Figure \ref{pass}.

\begin{figure}[ht]
\centering		
	\includegraphics[width=62.5mm]{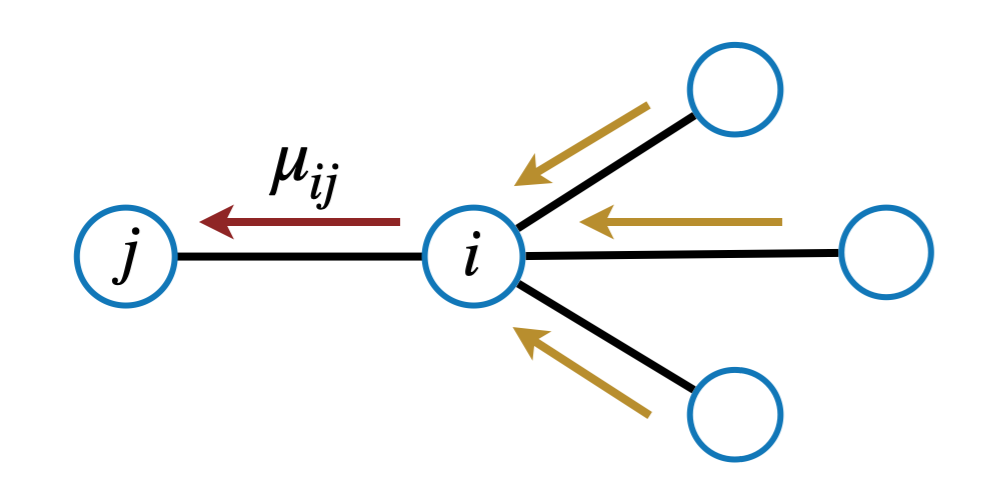}
	\caption{Illustration of message passing.  Node $i$ sends a message to a neighboring node $j$ that aggregates messages from all other neighbors of $i$.}
	\label{pass}
\end{figure}

Let $\mu_{ij}\in\mathbb{R}^m$ be the message sent from node $i$ to node $j$.  Each message is defined using the max-product equation,
\begin{equation*}
\mu_{ij}(x_j)= \max_{x_i}\big\{\,\phi_i(x_i)\,\psi_{ij}(x_i,x_j)\prod_{k\in N(i)\backslash j}\mu_{ki}(x_i)\, \big\}.
\end{equation*}
Intuitively, the message $\mu_{ij}$ provides information regarding what state node $i$ thinks node $j$ should be in, where large values of $\mu_{ij}(x_j)$ correspond to favorable states. The messages are often normalized for numerical reasons, in which case $\mu_{ij}$ is a distribution over the set of possible states. The incoming messages at each node are combined to compute a \emph{belief} function $b_j:\Omega\rightarrow\mathbb R_+$ given by
\begin{equation*}
b_{j}(x_j)=\phi_j(x_j)\prod_{i\in N(j)}\mu_{ij}(x_j).
\end{equation*}
For acyclic graphs the resulting beliefs are the exact max-marginals (up to a non-consequential scaling factor due to the normalization of messages), and the MAP solution can be obtained by maximizing each belief individually. 

When the graph contains cycles, this approach can be adapted into a fixed point iteration scheme referred to as \emph{loopy} belief propagation. In this setting, the message passing equations are used to define an operator that updates all messages in parallel until convergence (there are sequential variants as well). This operator acts on the space
\begin{equation*}
    \mathcal K=\bigotimes_{i\in V}\bigotimes_{j\in N(i)}\Delta_+^m,
\end{equation*}
where $\Delta^m_+$ is the positive $m$-simplex. Here $\mu \in \mathcal K$ is a vector of messages.  

\begin{defn}\label{T}
The operator $T:\mathcal K \rightarrow \mathcal K$ in the max-product algorithm is $T=N\hat{T}$.  $\hat{T}$ computes new messages using the max-product equation and $N$ normalizes each message.
\begin{equation*}
\big(\hat{T}\mu\big)_{ij}(x_j)=\max_{x_i}\big\{\,\phi_i(x_i)\,\psi_{ij}(x_i,x_j)\prod_{k\in N(i)\backslash j}\mu_{ki}(x_i)\, \big\}
\end{equation*}
\begin{equation*}
\big(N\mu\big)_{ij}(x_j) = \frac{\mu_{ij}(x_j)}{\sum\limits_\tau \mu_{ij}(\tau)}.
\end{equation*}
\end{defn}

The messages are initialized as $\mu^{(0)}\in\mathcal K$ and repeatedly updated via the fixed point iteration scheme $\mu^{(n+1)}=T\mu^{(n)}$.
This scheme is either run for a large number of iterations or stopped once the messages have sufficiently converged. The resulting messages $\mu^{(n)}$ are then used to compute beliefs $b_j^{(n)}$ for every node. 

\begin{defn}\label{def:blf}
The belief $b^{(n)}_j:\Omega\rightarrow\mathbb{R}$ of node $j\in V$ after $n$ iterations is given by
\begin{equation*}
    b^{(n)}_j(x_j)=\phi_j(x_j)\,\prod_{i\in N(j)}\mu^{(n)}_{ij}(x_j).
\end{equation*}
\end{defn}

BP has been successfully used in a variety of applications that involves cyclic graphs and often leads to good results. By treating the final beliefs as approximations to max-marginals we can obtain a labelling by maximizing each belief individually. However, BP is also known to fail to converge, have multiple fixed points, and for being sensitive to the message initialization.

One numerical method that improves the performance of BP is to stabilize the fixed point iteration scheme with damping\footnote{This scheme is referred to as the Krasnoselskij iteration scheme in the numerical analysis literature.}. This involves updating messages with the convex combination:
\begin{equation*}
    \mu^{(n+1)}:=(1-\alpha)\,T\mu^{(n)}+\alpha\,\mu^{(n)}
\end{equation*}
where $\alpha\in(0,1)$ is referred as the damping factor. In practice, damped BP (i.e. BP with damping) often prevents the messages from oscillating and converges faster than non-damped BP when both algorithms converge. However, even damped BP is still not guaranteed to converge.  We provide an example where damped BP does not converge below (Example \ref{ex:accuracy_ms}).

\subsection{Sum-Product Algorithm}

The sum-product BP algorithm is a variation of the max-product BP algorithm that can be used to compute marginal distributions.
The sum-product algorithm involves fixed point iteration with an operator that is nearly the same as in the max-product case. The only difference is that this operator includes a ``sum'' instead of a ``max'' in the message update equations.

\begin{defn}
The operator $T:\mathcal K\rightarrow\mathcal K$ in the sum-product algorithm is $T=N\hat T$.   $\hat T$ computes new messages using the sum-product equation and $N$ normalizes messages to sum to one.
\begin{equation*}
\big(\hat T\mu\big)_{ij}(x_j)=\sum_{x_i}\phi_i(x_i)\,\psi_{ij}(x_i,x_j)\prod_{k\in N(i)\backslash j}\mu_{ki}(x_i).
\end{equation*}
\begin{equation*}
\big(N\mu\big)_{ij}(x_j) = \frac{\mu_{ij}(x_j)}{\sum\limits_\tau \mu_{ij}(\tau)}.
\end{equation*}
\end{defn}

As in the max-product case messages are initialized to
arbitrary values and 
and repeatedly updated via fixed point iteration. 
After convergence the messages are used to compute beliefs for every node.  
In the sum-product algorithm the beliefs provide an approximation to the marginal distributions of each random variable. 

Similar to the max-product algorithm the sum-product algorithm performs exact inference in polynomial time when the graph is acyclic.  For cyclic graphs the sum-product algorithm is a heuristic.  The sum-product algorithm (even with damping) suffers from the exact same issues as the max-product algorithm.  Namely, the algorithm can fail to converge and can return different results depending on the initial set of messages.


\section{Convex Combination Belief Propagation}\label{sec:ccbp}

Now present \emph{convex combination belief propagation} (CCBP), a globally convergent alternative to belief propagation. The main objectives of this section are to introduce the new message passing operator, then prove that CCBP converges to a unique fixed point on graphs with arbitrary topology and arbitrary potential functions. This section focuses on the max-product version of CCBP, the sum-product version is analogous and described in Appendix B.

\subsection{Message Passing Operator}

Graphs with many short cycles are especially problematic for BP because they create feedback loops where information within messages is over-counted. As a result, the messages may oscillate, converge to inaccurate beliefs, or converge to different fixed points depending on the initialization. 

To define CCBP we take the operator from BP and weight the incoming messages when computing a new outgoing message.  When computing the message from node $i$ to node $j$ we weight the incoming message from another neighbor $k$ by $w_{ki}$.  The only conditions imposed upon these weights is that they are non-negative and sum to (at most) one.  We also discount all of the incoming messages by a factor $\gamma \in (0,1)$.

\begin{defn}\label{min_sum_op}
The operator $S:\mathcal{M}\rightarrow\mathcal{M}$ in the max-product CCBP algorithm is
\begin{equation*}
\big(S\mu\big)_{ij}(x_j)= \max_{x_i}\big\{\,\phi_i(x_i)\,\psi_{ij}(x_i,x_j) \Big(\prod_{k\in N(i)\backslash j}\mu_{ki}(x_i)^{w_{ki}} \Big)^\gamma\,\big\},
\end{equation*}
where the weights must be non-negative with $\sum\limits_{k\in N(i)\backslash j}w_{ki}\leq1$, and $\gamma\in(0,1)$. 
\end{defn}

Note that we do not incorporate normalization in the definition of the operator. The exclusion of the normalization factor implies that the operator acts on the space:
\begin{equation*}
    \mathcal M:=\bigotimes_{i\in V}\bigotimes_{j\in N(i)}\,\mathbb R^m_+
\end{equation*}

The simplest way to define each weight is to set them uniformily for each node based on the degree, $w_{ki}=1\slash\big(d(i)-1\big)$, where $d(i)$ is the degree of node $i$. Alternatively one can give more weight to some edges based on some additional information from a particular application (see \cite{felzenszwalb19}). Intuitively, the weights control how much influence neighboring nodes have upon each other. When the message sent from node $i$ to $j$ incorporates uniform weights, the other neighbors have equal influence upon node $j$. Non-uniform weights may be used to give some neighbors more influence. 

This operator also incorporates a damping factor $\gamma\in(0,1)$. Later we see that this term is the Lipschitz constant of the operator. Thus, the rate of convergence is dependent upon the magnitude of this parameter.  Next we present an example and numerical experiments that illustrates how the magnitude of $\gamma$ affects the performance of CCBP. 

\begin{exmp}\label{ex:gamma}
Let $G=(V,E)$ be an undirected graph with 10 nodes, where each pair of nodes is connected with probability 0.5. Let $\Omega =\{-1,1\}$ be the set of possible states for each random variable. Consider the joint distribution
\begin{equation*}
\mathbb P(X=x)=\frac{1}{Z}\exp\bigg(-\sum_{i\in V}x_iy_i-\sum_{\{i,j\}\in E}\lambda_{ij}x_i x_j\bigg)
\end{equation*}

We generated a concrete problem instance from this model by independently sampling $y_i$ from $\{-1,1\}$ and $\lambda_{ij}$ from a normal distribution. Then we repeatedly applied CCBP to this problem instance, while varying the magnitude of $\gamma$ from 0 to 0.9 in increments of $\Delta\gamma=0.1$. The weights $w_{ki}$ were set uniformly. The performance of the algorithm is determined by (1) computing the mean square error between the resulting beliefs after convergence and true max-marginals and (2) the number of iterations until convergence. 
\begin{figure}[H]
    \centering
    \includegraphics[width=155mm]{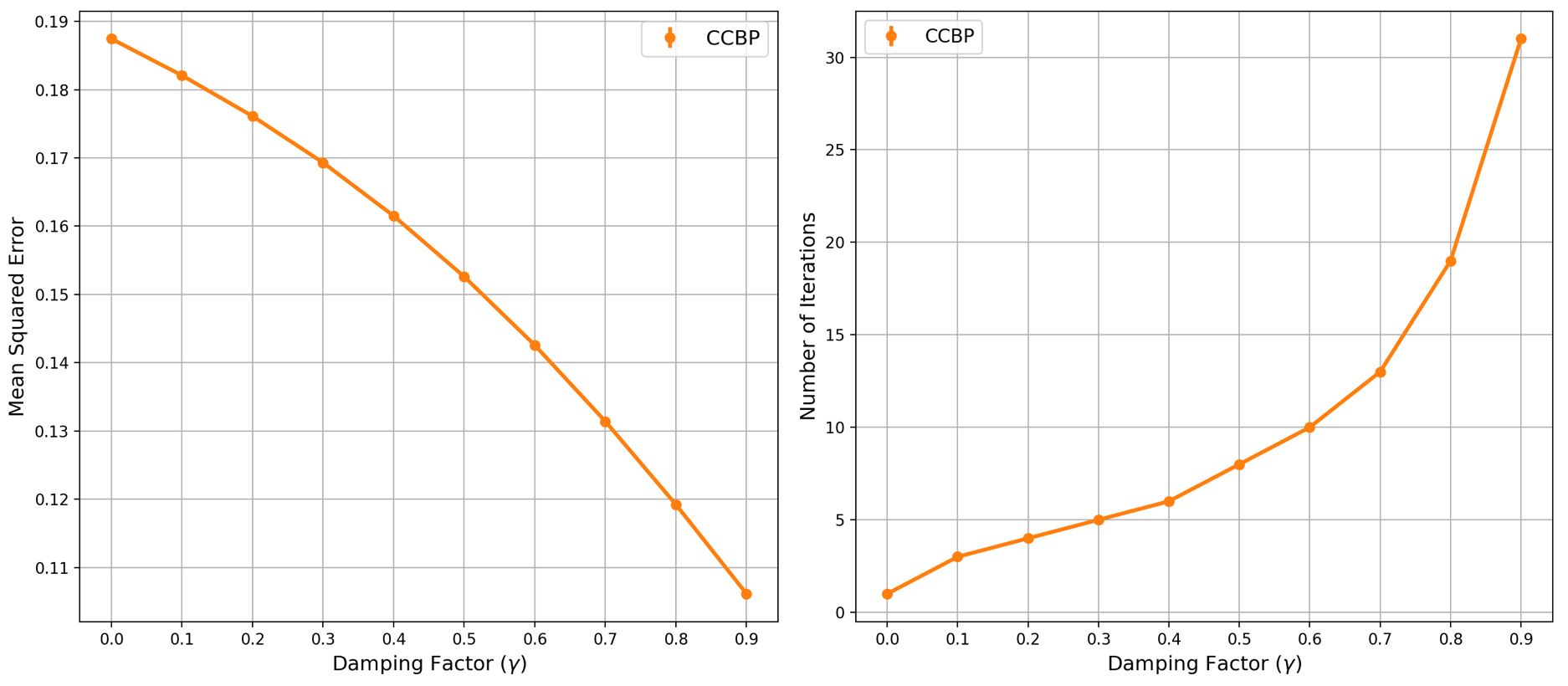}
    \caption{Results of experiment in Example \ref{ex:gamma}}
    \label{fig:gamma}
\end{figure}
\end{exmp}

The results from the numerical experiments with Example \ref{ex:gamma} are shown in Figure \ref{fig:gamma}. We see that the algorithm converges faster when the magnitude of $\gamma$ is smaller. However, the algorithm obtains a better approximation of the true max-marginals for larger values of $\gamma$. We recommend setting this factor close to one (i.e. $\gamma=0.9)$ since the algorithm  converges very quickly.  Note that the algorithm is guaranteed
to converge for any value $\gamma < 1$.

\subsection{Convergence}\label{sec:conv}

Next we prove that CCBP converges to a unique fixed point independent of the topology of the graph and message initialization. Our general approach is to consider CCBP as a discrete-time map, then use results from the theory of dynamical systems. In particular, we invoke Banach's fixed point theorem which guarantees existence and uniqueness of fixed points of certain discrete-time maps referred to as \emph{contractions}. Given a metric space $(X,d)$, an operator $F:X\rightarrow X$ is called a contraction if there is a constant $\gamma\in(0,1)$ such that
\begin{equation*}
d(F(x),F(y))\leq\gamma d(x,y)
\end{equation*}
for any $x,y\in X$. Banach's fixed point theorem states that a contractive operator (i.e. contraction) defined over a complete metric space has a unique and globally attractive fixed point $x^\star\in X$ \cite{granas03}. In order to use this result, we must define a metric which is complete with respect to the space $\mathcal M$.

\begin{prop}\label{complete}
Let $d:\mathcal M\times\mathcal M\rightarrow\mathbb R$ be the distance function given by
\begin{equation*}
d(\mu,\nu)=\max_{i\in V}\max_{j\in N(i)}\max_{x_j}\big\vert\log\,\mu_{ij}(x_j)-\log\,\nu_{ij}(x_j)\big\vert,
\end{equation*}
then the pair $\big(\mathcal M,d\big)$ is a complete metric space. 
\end{prop}
\begin{proof}
See Appendix A for proof.
\end{proof}

The metric $d$ computes the distance between individual messages $\mu,\nu\in\mathcal M$. 
Next we prove that $S$ is a contraction with respect to this metric. 

\begin{lemma}\label{contraction}
The operator $S$ is a contraction with Lipschitz constant $\gamma$.
\end{lemma}
\begin{proof}
Choose any $\mu,\nu\in\mathcal{M}$, then
\begin{align*}
\big(S\mu\big)_{ij}(x_j)&=\max_{x_i}\big\{\,\phi_i(x_i)\,\psi_{ij}(x_i,x_j)\prod_{k\in N(i)\backslash j}\mu_{ki}(x_i)^{^\gamma\,w_{ki}}\,\big\} \\
&= \max_{x_i}\Big\{\,\phi_i(x_i)\,\psi_{ij}(x_i,x_j)\prod_{k\in N(i)\backslash j}\nu_{ki}(x_i)^{\gamma\,w_{ki}}\,\prod_{k\in N(i)\backslash j}\frac{\mu_{ki}(x_i)^{\gamma\,w_{ki}}}{\nu_{ki}(x_i)^{\gamma\,w_{ki}}}
\,\Big\} \\
&\leq \max_{x_i}\big\{\,\phi_i(x_i)\,\psi_{ij}(x_i,x_j)\prod_{k\in N(i)\backslash j}\nu_{ki}(x_i)^{\gamma\,w_{ki}}\,\}\; \max_{x_i}\prod_{k\in N(i)\backslash j}\frac{\mu_{ki}(x_i)^{\gamma\,w_{ki}}}{\nu_{ki}(x_i)^{\gamma\,w_{ki}}}\\
&=\big(S\nu\big)_{ij}(x_j)\;\max_{x_i}\prod_{k\in N(i)\backslash j}\frac{\mu_{ki}(x_i)^{\gamma\,w_{ki}}}{\nu_{ki}(x_i)^{\gamma\,w_{ki}}}
\end{align*}
Given that both sides of the inequality are positive, we apply the logarithm to obtain:
\begin{align*}
    \log\big(S\mu\big)_{ij}(x_j)-\log\big(S\nu\big)_{ij}(x_j)&\leq \gamma\sum_{k\in N(i)\backslash j}w_{ki}\max_{x_i}\,\big\vert\log\mu_{ki}(x_i)-\log\nu_{ki}(x_i)\,\big\vert\\
    &\leq\gamma\max_{k\in N(i)\backslash j}\max_{x_i}\,\big\vert\log\mu_{ki}(x_i)-\log\nu_{ki}(x_i)\,\big\vert
\end{align*}
Since this inequality holds when $\mu$ and $\nu$ are interchanged, it also holds with respect to the absolute value of the left hand side. In addition, this inequality holds for any $x_j\in\Omega$ and so it holds for the maximum of the left hand side. 
\begin{align*}
    \max_{x_i}\big\vert\,\log\big(S\mu\big)_{ij}(x_j)-\log\big(S\nu\big)_{ij}(x_j)\,\big\vert\
    &\leq\gamma\max_{k\in N(i)\backslash j}\max_{x_i}\big\vert\log\mu_{ki}(x_i)-\log\nu_{ki}(x_i)\,\big\vert \\
    &\leq\gamma\max_{i\in V}\max_{j\in N(i)}\max_{k\in N(i)\backslash j}\max_{x_i}\big\vert\log\mu_{ki}(x_i)-\log\nu_{ki}(x_i)\,\big\vert \\
    &=\gamma\max_{i\in V}\max_{j\in N(i)}\max_{x_j}\big\vert\log\mu_{ij}(x_j)-\log\nu_{ij}(x_j)\,\big\vert \\
    &=\gamma\, d(\mu,\nu)
\end{align*}
Since the inequality holds for any $\{i,j\}\in E$ on the left hand side we conclude
\begin{equation*}
d\big(S\mu, S\nu\big)\leq\gamma\, d(\mu,\nu).
\end{equation*}
\end{proof}

Now our main theorem follows directly from the contraction result.

\begin{thm}\label{convergence}
The operator $S$ has a unique fixed point $\mu^\star\in\mathcal M$ and any sequence of messages defined by $\mu^{(n+1)}:=S\mu^{(n)}$ converges to $\mu^\star$. Furthermore, after $n$ iterations
\begin{equation*}
d\big( S^{(n)}\mu^{(0)},\mu^\star \big) \leq \gamma^n\,d(\mu^{(0)},\mu^\star\big).
\end{equation*}
\end{thm}
\begin{proof}
$S$ is a contraction with respect to the metric $d$ according to Lemma \ref{contraction} and the space $\mathcal M$ is complete with respect to $d$.  The result holds by applying Banach's fixed point theorem. 
\end{proof}

This result shows that CCBP is guaranteed to converge when the weights satisfy the simple constraints stated in the definition of the operator.  Note that the constraints we have
on the weights do not depend on the strength
of the potential functions in the graphical
model.  The weights can also be easily
selected based on the degree of each node.

Finally we note that the CCBP operator has a unique fixed point once the weights $w_{ki}$ and $\gamma$ are fixed, but the choice of weights and $\gamma$ do affect the resulting fixed point.   A simple choice for these paramaters involves setting $w_{ki}=1\slash\big(d(i)-1\big)$ and $\gamma$ to be a value close to 1.

\subsection{Characterization of Beliefs}

Although CCBP is guaranteed to converge, the algorithm is only useful if it returns a meaningful result. When the underlying graph is acyclic the traditional max-product algorithm computes the exact max-marginals. In this section, we provide an analogous characterization of the beliefs obtained with CCBP. 

\subsubsection{Min-Sum Algorithm}

CCBP along with other related methods improve convergence by changing the optimization landscape. For our purposes, it is more natural to discuss this optimization problem in terms of minimizing an energy as opposed to maximizing a joint distribution. Given that the potentials are assumed to be positive, the joint distribution can be equivalently written as
\begin{equation*}
    \mathbb P(X=x)=\frac{1}{Z}e^{-E(x)},
\end{equation*}
where $E$ is an energy. In the case of a pairwise model, this energy is
\begin{equation}\label{eq:min-sum}
    E(x)=\sum_{i\in V}g_i(x_i)+\sum_{\{i,j\}\in E}h_{ij}(x_i,x_j),
\end{equation}
where $g_i(x_i)=-\log\,\phi_i(x_i)$ and $h_{ij}(x_i,x_j)=-\log\,\psi_{ij}(x_i,x_j)$ are cost functions.  With these definitions  the most probable (MAP) solution is the configuration of the random vector that minimizes the energy $E(x)$. 

The max-product CCBP algorithm can be implemented with negative log probabilities, where the max-product becomes a min-sum. The equivalent message passing operator is
\begin{equation*}
    (S\mu)_{ij}(x_j)=\min_{x_i}\big\{\,g_i(x_i)+h_{ij}(x_i,x_j)+\gamma\sum_{k\in N(i)\backslash j}w_{ki}\,\mu_{ki}(x_i)\big\},
\end{equation*}
In this case, the resulting beliefs provide an approximation of the \emph{min-marginals} of the energy. For the remainder of this section, we focus on the min-sum formulation.

\subsubsection{Tree-Structured Graphs}

Let $\mu$ be the unique fixed point of $S$, the belief $b_j:\Omega\rightarrow\mathbb{R}$  is
\begin{equation*}
b_j(x_j)=g_j(x_j)+\sum_{i\in N(j)}\mu_{ij}(x_j).
\end{equation*}
The main objective of this section is prove that these beliefs are the exact min-marginals of a weighted energy $E_j(x)$.  This energy is closely related to $E(x)$, but a key difference is that each cost function is weighted according to the weights and damping factor used in the definition of $S$. We prove that each belief $b_j$ is the min-marginal of a different energy function $E_j(x)$ as the weights of each term depend on the choice of $j$.

Before defining the weighted energy for an arbitrary tree-structured graph we provide a simple example of how to write down this energy for a small graph. In order to simplify the expressions for the energies, we define the cost function $H_{ij}(x_i,x_j)=g_i(x_i)+h_{ij}(x_i,x_j)$.

\begin{exmp}
Let $G$ be the graph shown in Figure \ref{ex_acyc}.  In this case, the energy $E$ is given by
\begin{equation*}
E(x)=g_1(x_1)+H_{21}(x_2,x_1)+ H_{32}(x_3,x_2)+H_{43}(x_4,x_3)+H_{53}(x_5,x_3).
\end{equation*} 

Let $w_{ki}=1\slash\big(d(i)-1\big)$ in the definition of $S$.  When CCBP is applied to this graph, the resulting belief function $b_1$ is the min-marginal of the energy: 
\begin{equation*}
E_1(x)=g_1(x_1)+H_{21}(x_2,x_1)+\gamma H_{32}(x_3,x_2)+\frac{\gamma^2}{2}H_{43}(x_4,x_3)+\frac{\gamma^2}{2}H_{53}(x_5,x_3),
\end{equation*} 
This energy contains the same cost functions as the energy $E$, but the difference is that each term $H_{ij}(x_i,x_j)$ is multiplied by a product of weights and a power of $\gamma$. 
\begin{figure}[H]
	\centering		
	\includegraphics[width=70mm]{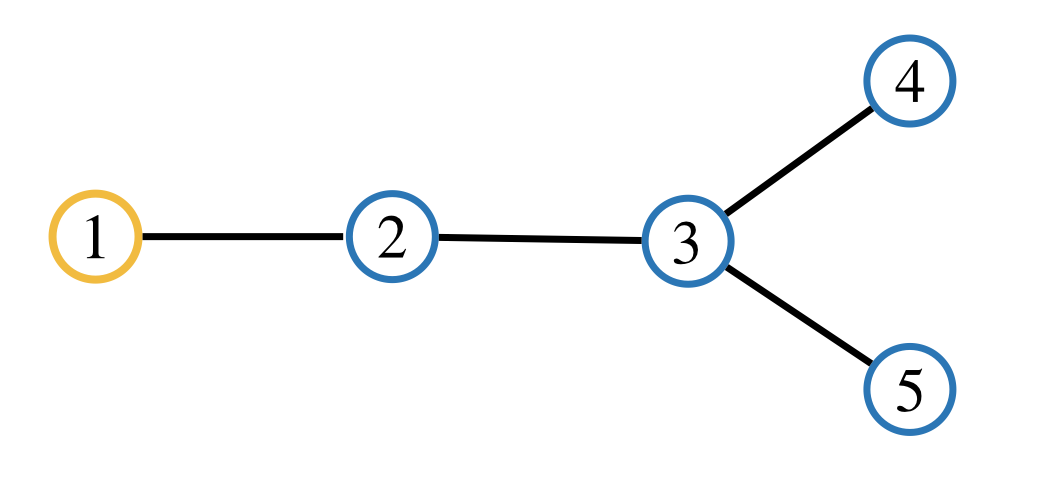}
	\caption{A simple tree-structured graph.}
	\label{ex_acyc}
\end{figure}
\end{exmp}

Next we generalize the energy in the example above to the case of an arbitrary acyclic graph.    
Let $\mathcal T (j)$ be a tree-structured graph with a distinguished root node $j\in V$.
Once a tree $\mathcal T$ is rooted at a node $j$, every node $i \neq j$ has a unique parent $\mathcal{P}(j,i)$, a set of children $\mathcal C(j,i)$, and a set of descendants $\mathcal{D}(j,i)$.  Let $\mathcal T(j,i)$ denote the subtree of $\mathcal T(j)$
rooted at a node $i$ as illustrated in Figure \ref{root}.   The subtree $\mathcal T(j,i)$ includes node $i$ and its descendants. Let $R(\mathcal T(j))$ be the depth of a rooted tree.  Let $N_d(\mathcal T(j))\subseteq V$ with $d\ge0$ be the set of nodes at distance $d$ from the root $j$.

\begin{figure}[H]
\centering		
	\includegraphics[width=110mm]{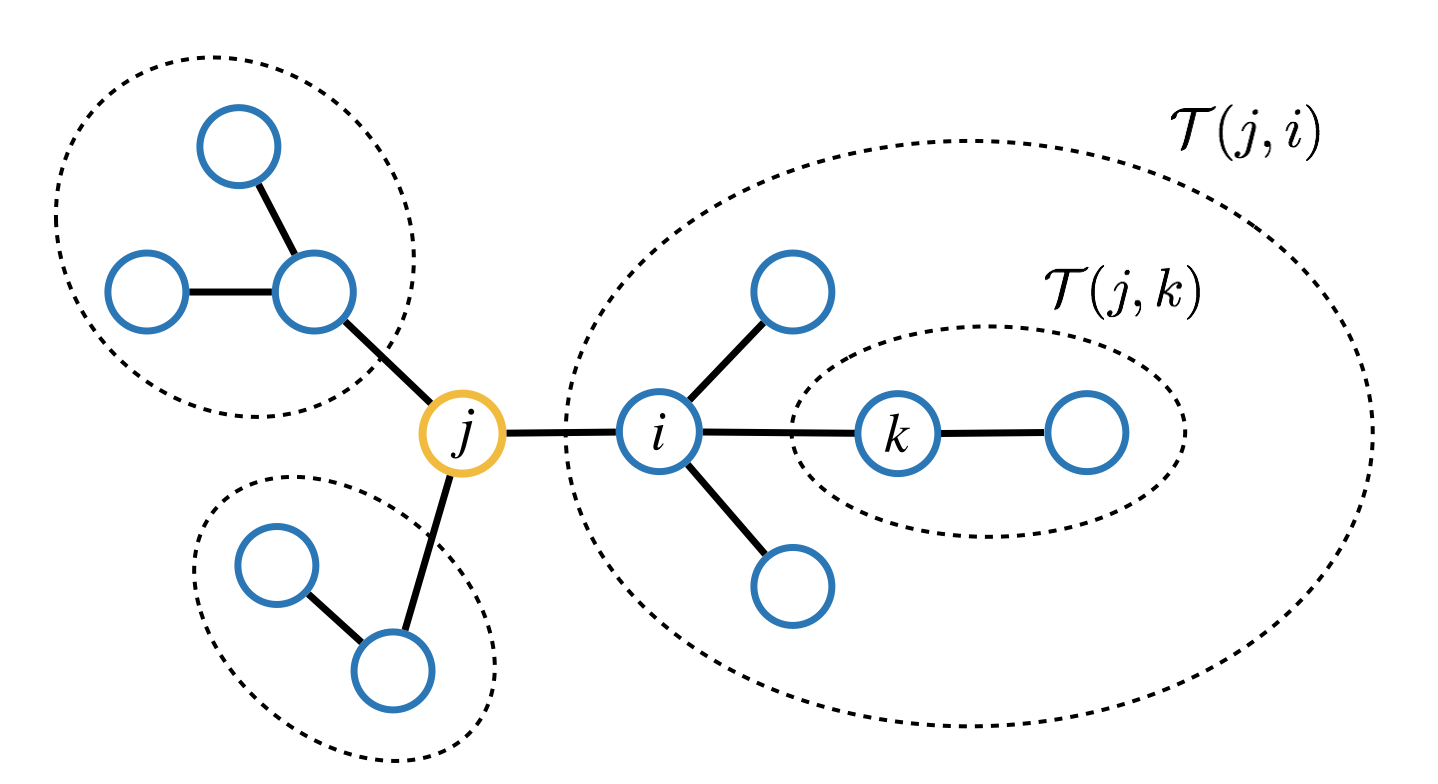}
	\caption{Subtrees of $\mathcal T(j)$}
	\label{root}
\end{figure}

\begin{defn}
Let $w:V\times V \rightarrow \mathbb{R}$ be the weight function 
\begin{equation*}
w(i,j)=\prod_{(k,\ell)\in E(i,j)}w_{k\ell},
\end{equation*}
where $E(i,j)$ is the collection of directed edges along the path from node $i$ to $j$.
\end{defn}

\begin{defn}\label{energy_trees}
Let $E_j:\Omega^n\rightarrow \mathbb{R}$ be the weighted energy given by
\begin{equation*}
E_j(x)=g_j(x_j)+\sum_{i \in N(j)} \sum_{d=0}^{R(\mathcal T(j,i))-1}\sum_{k\in N_d(\mathcal T(j,i))} \gamma^d w(k,i) H_{k\mathcal P(j,k)}(x_k,x_{\mathcal P(j,k)}),
\end{equation*}
\end{defn} 

This energy is closely related to $E(x)$ but it puts more weight on terms ``near'' node $j$.  For each edge $(k,\mathcal P(j,k))$ the cost associated with that edge is scaled by $\gamma^d w(k,i)$.  When the damping factor and weights are removed from $E_j$, this energy is exactly $E(x)$.
\begin{prop}
If $w_{ki}=1$ for all $\{k,i\}\in E$ and $\gamma=1$, then $E_j(x)=E(x)$. 
\end{prop}
\begin{proof}
\begin{align*}
E_j(x)&=g_j(x_j)+\sum_{i \in N(j)} \sum_{d=0}^{R(\mathcal T(j,i))-1}\sum_{k\in N_d(\mathcal T(j,i))} \gamma^d w(k,i) H_{k\mathcal P(j,k)}(x_k,x_{\mathcal P(j,k)}) \\
&=g_j(x_j)+\sum_{i \in N(j)} \sum_{d=0}^{R(\mathcal T(j,i))-1}\sum_{k\in N_d(\mathcal T(j,i))}\Big(g_k(x_k)+h_{k\mathcal P(j,k)}\big(x_k,x_{\mathcal P(j,k)}\big)\Big) \\
&=g_j(x_j)+\sum_{k\in V\backslash j}\Big(g_k(x_k)+h_{k\mathcal P(j,k)}\big(x_k,x_{\mathcal P(j,k)}\big)\Big) \\
&=\sum_{k\in V}g_k(x_k)+\sum_{\{k,i\}\in E}h_{ki}(x_k,x_i) \\
&=E(x).
\end{align*}
\end{proof}

Next, we prove that CCBP computes the exact min-marginals of the weighted energy in Definition \ref{energy_trees}. The main idea of the argument is to use that computing the min-marginals in a tree-structured graphical model can be broken down into subproblems that can be solved recursively. Each subproblem involves a function $F_{ji}:\Omega\rightarrow\mathbb{R}$ that corresponds to the min-marginals of an optimization problem defined by $\mathcal T(j,i)$.

\begin{lemma}\label{lemma:F_ji}
Let  $F_{ji}:\Omega\rightarrow\mathbb R$ be the function given by
\begin{equation*}
F_{ji}(x_i)=\min_{x_{\mathcal D(j,i)}}\Big\{ \sum_{d=1}^{R(\mathcal T(j,i))-1} \sum_{k\in N_d(\mathcal T(j,i))}\gamma^d w(k,i)\, H_{k\mathcal P(j,k)}\big(x_k,x_{\mathcal P(j,k)}\big) \Big\},
\end{equation*}
then $F_{ji}(x_i)=\gamma \sum\limits_{k\in N(i)\backslash\mathcal P(j,i)} w_{ki}\mu_{ki}(x_i)$ where $\mu$ is the fixed point of $S$.
\end{lemma}
\begin{proof}
This claim can be proven by inducting on the depth of the subtree $\mathcal T(j,i)$. In the base case when the depth is 1, then $F_{ji}(x_i) = 0$ and the result follows trivially.
 
Now for the induction step,
\begin{align*}
F_{ji}(x_i)&=\min_{x_{\mathcal D(j,i)}}\Big\{ \sum_{d=1}^{R(\mathcal T(j,i))-1}\sum_{k\in N_d(\mathcal T(j,i))}\gamma^d\,w(k,i)H_{k\mathcal P(j,k)}\big(x_k,x_{\mathcal P(j,k)}\big) \Big\}\\
&=\min_{x_{\mathcal D(j,i)}}\Big\{\sum_{k\in N(i)\backslash \mathcal P(j,i)}\gamma w_{ki}H_{ki}(x_k,x_i)+\sum_{d=2}^{R(\mathcal T(j,i))-1} \sum_{k^\prime\in N_d(\mathcal T(j,i))}\gamma^d w(k^\prime,i)\, H_{k^\prime\mathcal P(j,k^\prime)}\big(x_{k^\prime},x_{\mathcal P(j,k^\prime)}\big)\Big\} \\
&=\min_{x_{\mathcal D(j,i)}}\Big\{\sum_{k\in N(i)\backslash \mathcal P(j,i)}\gamma w_{ki}\Big(H_{ki}(x_k,x_i)+\sum_{d=1}^{R(\mathcal T(j,k))-1} \sum_{k^\prime\in N_d(\mathcal T(j,k))}\gamma^d w(k^\prime,k)\, H_{k^\prime\mathcal P(j,k^\prime)}\big(x_{k^\prime},x_{\mathcal P(j,k^\prime)}\big)\Big)\Big\} \\
&=\gamma \sum_{k\in N(i)\backslash \mathcal P(j,i)} w_{ki}\min_{x_k} \Big\{H_{ki}(x_k,x_i)+ \min_{x_{\mathcal D(j,k)}} \Big\{\sum_{d=1}^{R(\mathcal T(j,k))-1} \sum_{k^\prime\in N_d(\mathcal T(j,k))}\gamma^d w(k^\prime,k)\, H_{k^\prime\mathcal P(j,k^\prime)}\big(x_{k^\prime},x_{\mathcal P(j,k^\prime)}\big)\Big\}\Big\} \\
&=\gamma \sum_{k\in N(i)\backslash \mathcal P(j,i)} w_{ki}\min_{x_k} \Big\{H_{ki}(x_k,x_i)+F_{jk}(x_k)\Big\} \\
&=\gamma \sum_{k\in N(i)\backslash \mathcal P(j,i)} w_{ki}\min_{x_k} \Big\{H_{ki}(x_k,x_i)+ \gamma \sum_{k' \in N(k)\backslash\mathcal P(j,k)} w_{k'k} \mu_{k'k}(x_k) \Big\}\\
&=\gamma \sum_{k\in N(i)\backslash \mathcal P(j,i)}  w_{ki} \mu_{ki}(x_i).
\end{align*}
\end{proof}

\begin{thm}\label{char}
The belief $b_j$ is the exact min-marginal of $E_j$ with respect to the $j$-th variable.
\end{thm}
\begin{proof}
Choose any $j\in V$, then the min-marginal $p_j$ of $E_j$ is
\begin{align*}
p_j(\tau)&=\min_{\{x\,:\,x_j=\tau\}} E_j(x)\\
&=g_j(\tau)+\min_{\{x\,:\,x_j=\tau\}} \Big\{\sum_{i \in\mathcal N(j)} \sum_{d=0}^{R(\mathcal T(j,i))-1}\sum_{k\in\mathcal N_d(\mathcal T(j,i))} \gamma^d w(k,i) H_{k\mathcal P(j,k)}(x_k,x_{\mathcal P(j,k)})\Big\} \\
&=g_j(\tau)+\sum_{i\in\mathcal N(j)}\min_{x_i}\Big\{H_{ij}(x_i,\tau)+ \min_{x_{\mathcal D(j,i)}}\Big\{\sum_{d=1}^{R(\mathcal T(j,i))-1} \sum_{k\in\mathcal N_d(\mathcal T(j,i))}\gamma^d w(k,i)\, H_{k\mathcal P(j,k)}(x_{k},x_{\mathcal P(j,k)})\Big\}\Big\}
\end{align*}
Now Lemma \ref{lemma:F_ji} implies that
\begin{align*}
p_j(\tau)&=g_j(\tau)+\sum_{i\in N(j)}\min_{x_i}\Big\{H_{ij}(x_i,\tau)+ F_{ji}(x_i)\Big\}\\
&=g_j(\tau)+\sum_{i\in N(j)}\min_{x_i}\Big\{H_{ij}(x_i,\tau)+\gamma \sum_{k\in N(i)\backslash j}w_{ki}\mu_{ki}(x_i)\Big\} \\
&=g_j(\tau)+\sum_{i\in N(j)}\mu_{ij}(\tau) \\
&=b_j(\tau).
\end{align*}
\end{proof}


\section{Numerical Experiments}\label{sec:exps}

In this section, we present a series of numerical experiments that evaluate the performance of CCBP. First, we use the sum-product version of BP and CCBP to perform inference in the spin glass model.  We compare the results obtained by these algorithms to the exact marginals computed by brute force.  We also demonstrate the practicality of the max-product CCBP algorithm by using it to restore a noisy image.

\subsection{Spin Glass Model}

The objective of this experiment is to evaluate the performance of CCBP on inference tasks where traditional BP fails to converge. We apply each algorithm to a large number of problems, then evaluate the results with several performance metrics. 

\subsubsection{Experimental Settings}\label{sec:exp_settings}

The main idea behind this experiment is evaluate our algorithm on random graphical models while varying some parameter which influences the difficulty of performing inference.  The purpose of this section is to describe how problems are generated and define the performance metrics used to evaluate the algorithms. 

Let $G=(V,E)$ be an undirected graph.  Let $\Omega = \{-1,1\}$ be the set of possible outcomes for each random variable. The probability of a configuration of the random variables is given by
\begin{equation*}
\mathbb P(X=x)=\frac{1}{Z}\exp\bigg(-\sum_{i\in V} x_i y_i -\sum_{\{i,j\}\in E}\lambda_{ij}x_i x_j\bigg).
\end{equation*} 

In each problem instance, we generate a graph with 10 nodes such that the probability of two nodes being connected is 0.5.  We refer to this value as the edge appearance probability. Each $y_i$ is  uniformily sampled from $\{-1,1\}$ and the $\lambda_{ij}$ are independently sampled from a uniform distribution The magnitude of $\lambda_{ij}$ controls how strongly neighboring states are coupled and the sign determines whether neighbors prefer to be aligned or misaligned.

For each algorithm, we use the initialization $\mu^{(0)}_{ij}=(1,1)$ for all $\{i,j\}\in E$ and update the messages in parallel. Let $N=10^3$ be the maximum number of iterations. Let $\epsilon=10^{-2}$ be a threshold that indicates when the messages have sufficiently converged. For traditional BP, we use the damping factor $\alpha=0.9$ to prioritize convergence since we consider difficult problem instances. We use the damping factor $\gamma=0.9$ in CCBP and set the weights uniformly.

The performance of each algorithm is determined by the accuracy of the resulting beliefs, runtime, and rate of convergence. The accuracy of the beliefs is determined by computing the measure square error (MSE) between the normalized beliefs and the true marginals,
\begin{equation*}
    \text{MSE }=\frac{1}{n}\sum_{i}\sum_{x_i}\vert \,\hat{b}_i(x_i)-p_i(x_i)\,\vert^2,
\end{equation*}
where $p_i$ is the exact marginal of node $i\in V$ (computed by brute force). The runtime refers to the number of iterations until convergence the algorithm sufficiently converges. We only report the runtime for instances where the algorithm converges.

\subsubsection{Coupling Factor}

BP is known to fail when the message passing operator has repulsive fixed points. This complicates fixed point iteration because the scheme becomes unstable and is unlikely to converge for any damping factor. This scenario can be simulated by sampling the coupling factors $\lambda_{ij}$ from a uniform distribution centered about zero. In this case, repulsive fixed point are likely to emerge because the coupling factors are both positive and negative. 

In this experiment, we consider models where $\lambda_{ij}\sim\text{Unif}(-\sigma,\sigma)$ with $\sigma\geq0$. The magnitude of $\sigma$ was varied from 0 to 5 in increments of $\Delta\sigma=0.5$. For each value of $\sigma$, we generate 100 graphical models. The results of this experiment in terms of the previously defined performance metrics are shown in Figures \ref{fig:sigma-runtime} and \ref{fig:sigma-mse}. Note that each data point represents the average over all runs corresponding to a given value of $\sigma$.

\begin{figure}[H]
    \centering
    \includegraphics[width=155mm]{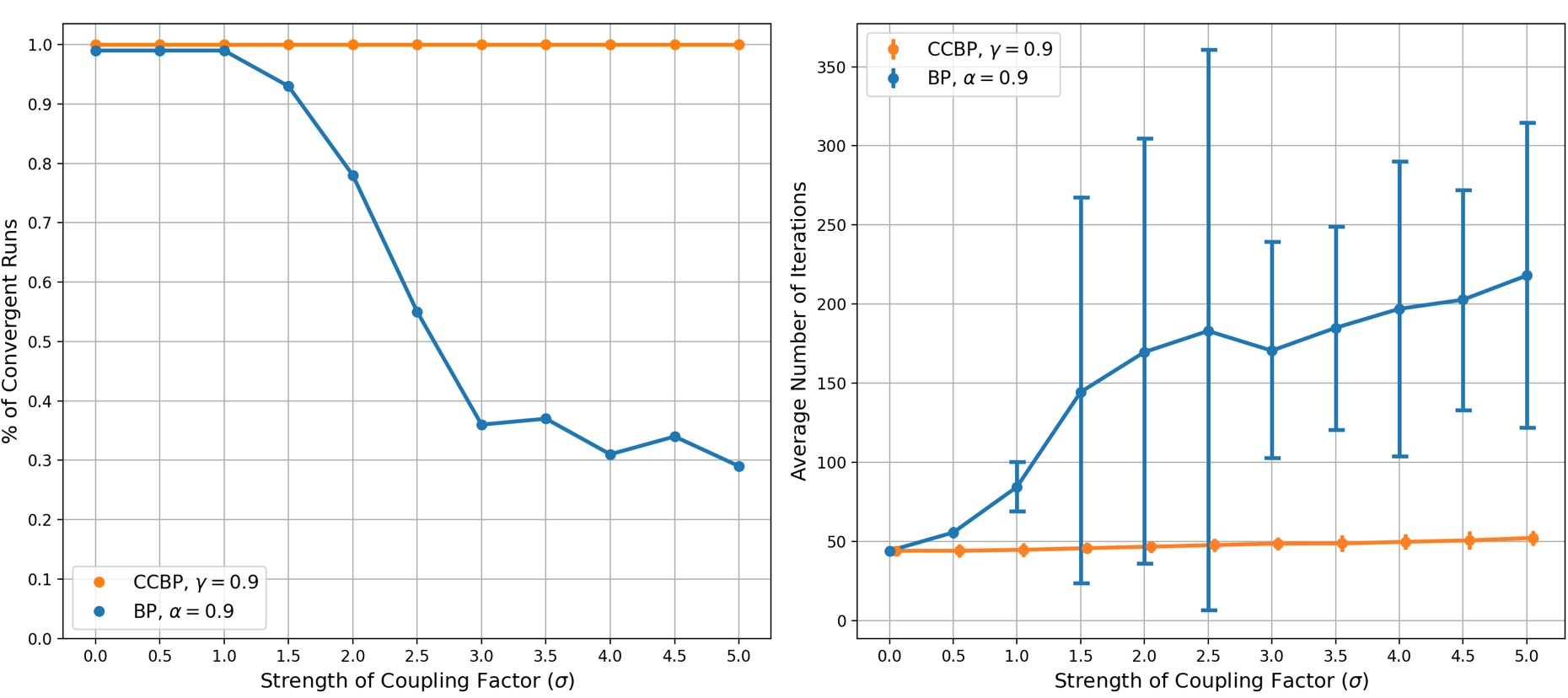}
    \caption{Convergence rate and runtime of BP algorithms}
    \label{fig:sigma-runtime}
\end{figure}
\begin{figure}[H]
    \centering
    \includegraphics[width=155mm]{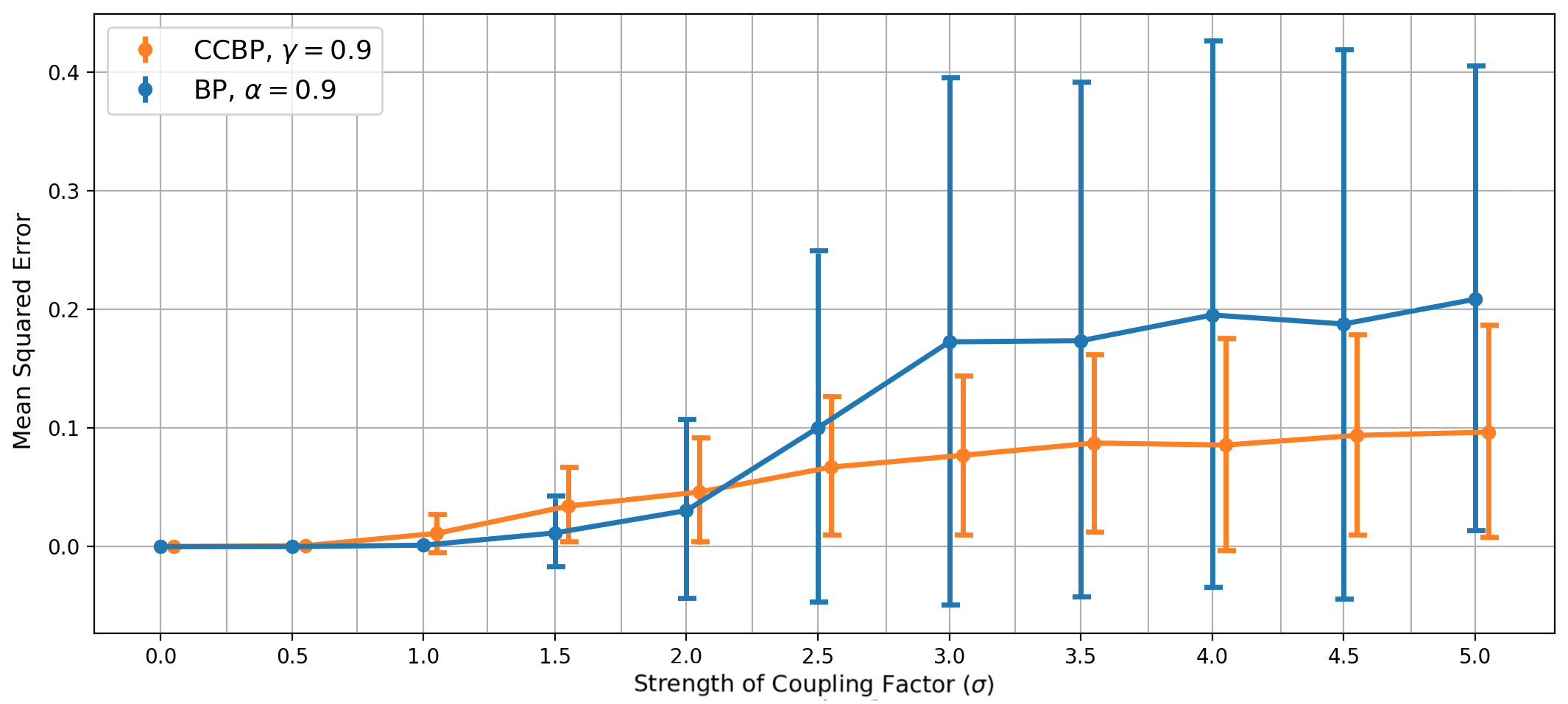}
    \caption{Accuracy of BP algorithms}
    \label{fig:sigma-mse}
\end{figure}

The results shown in Figure \ref{fig:sigma-runtime} and \ref{fig:sigma-mse} illustrate some key advantages of CCBP. First, we see that the algorithm always converges as expected due the result in Theorem \ref{convergence}. Moreover, we see that the runtime of the algorithm is relatively constant even as the strength of the coupling factor increases. This behavior is also expected due the operator being a contraction. In this case, the distance between the unique fixed point and current set of messages decreases by a factor of $\gamma$ on each iteration (see Theorem \ref{convergence}). 

One key distinction between CCBP and other alternative message passing algorithms is that we apply weights to the messages instead of the potentials. In previous methods weights were applied to potentials so that the strength of the coupling factor is sufficiently small to ensure convergence. This idea traces back to previous works that aim to derive criteria that guarantees when BP converges (e.g. Dobrushin's condition, see \cite{tatikonda03}, \cite{tatikonda02}). One drawback of this approach is that the weights must be smaller as the strength of the potentials increases. One important advantage of CCBP is that the condition on the weights is independent of the strength of the potentials. In the experiments described here the weights were set uniformly based on the degree of each node.

\subsubsection{Edge Connectivity}

BP is also known to either fail to converge or return a poor approximation when the topology of the graph is complex. In particular, the performance of the algorithm suffers when the graph contains many cycles because this leads to feedback loops. In this experiment, we compare the performance of the belief propagation algorithms while gradually increasing the connectivity of the graph. 

We consider graphical models where the edge appearance probability $p$ is varied from 0 to 1 with increments of $\Delta p =0.1$. For each value of $p$, we obtain 100 problem instances by generating a random graph and independently sampling $\lambda_{ij}\sim$ Unif($-5,5)$ for each problem instance. The results of this experiment are shown in Figures \ref{fig:edges-runtime} and \ref{fig:edges-mse}. 

\begin{figure}[H]
    \centering
    \includegraphics[width=155mm]{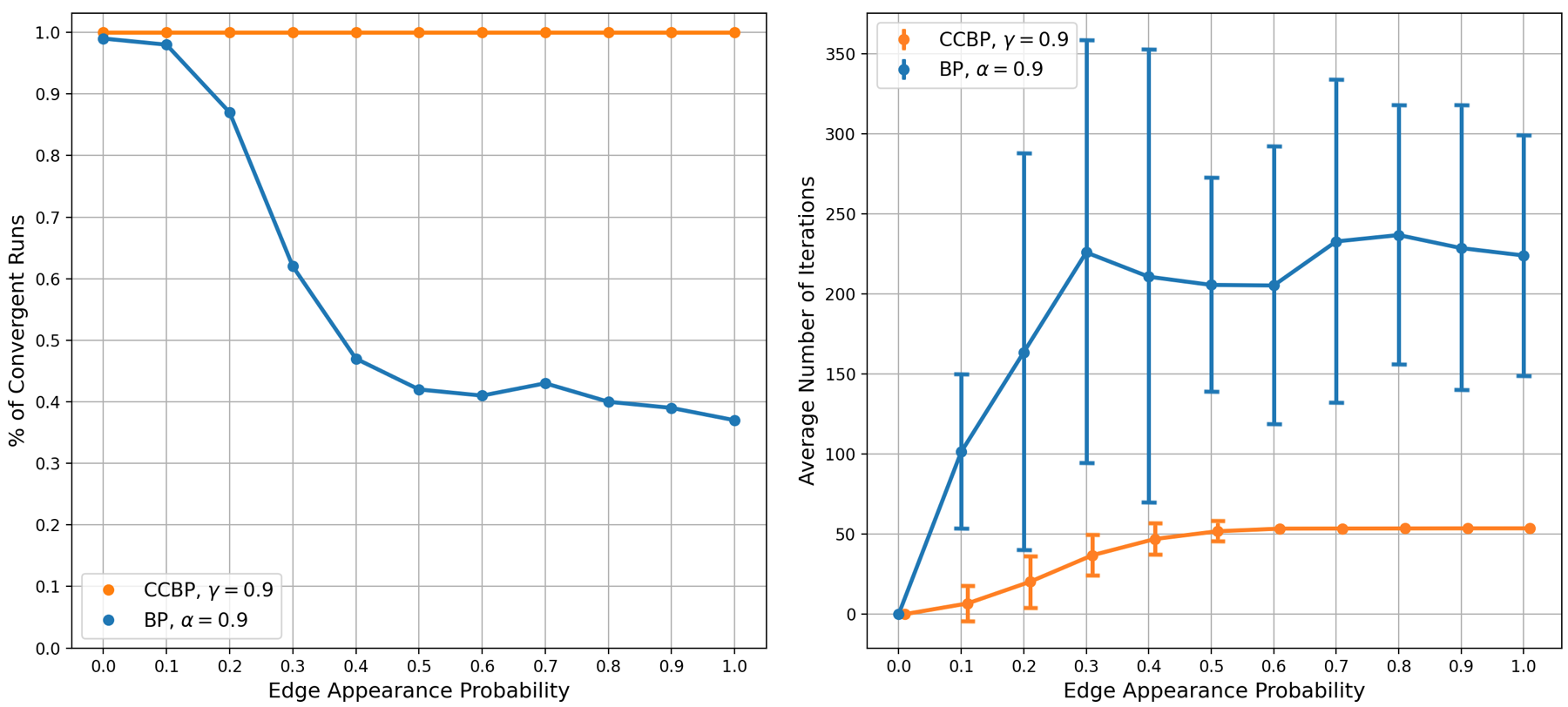}
    \caption{Convergence rate and runtime of BP algorithms}
    \label{fig:edges-runtime}
\end{figure}
\begin{figure}[H]
    \centering
    \includegraphics[width=155mm]{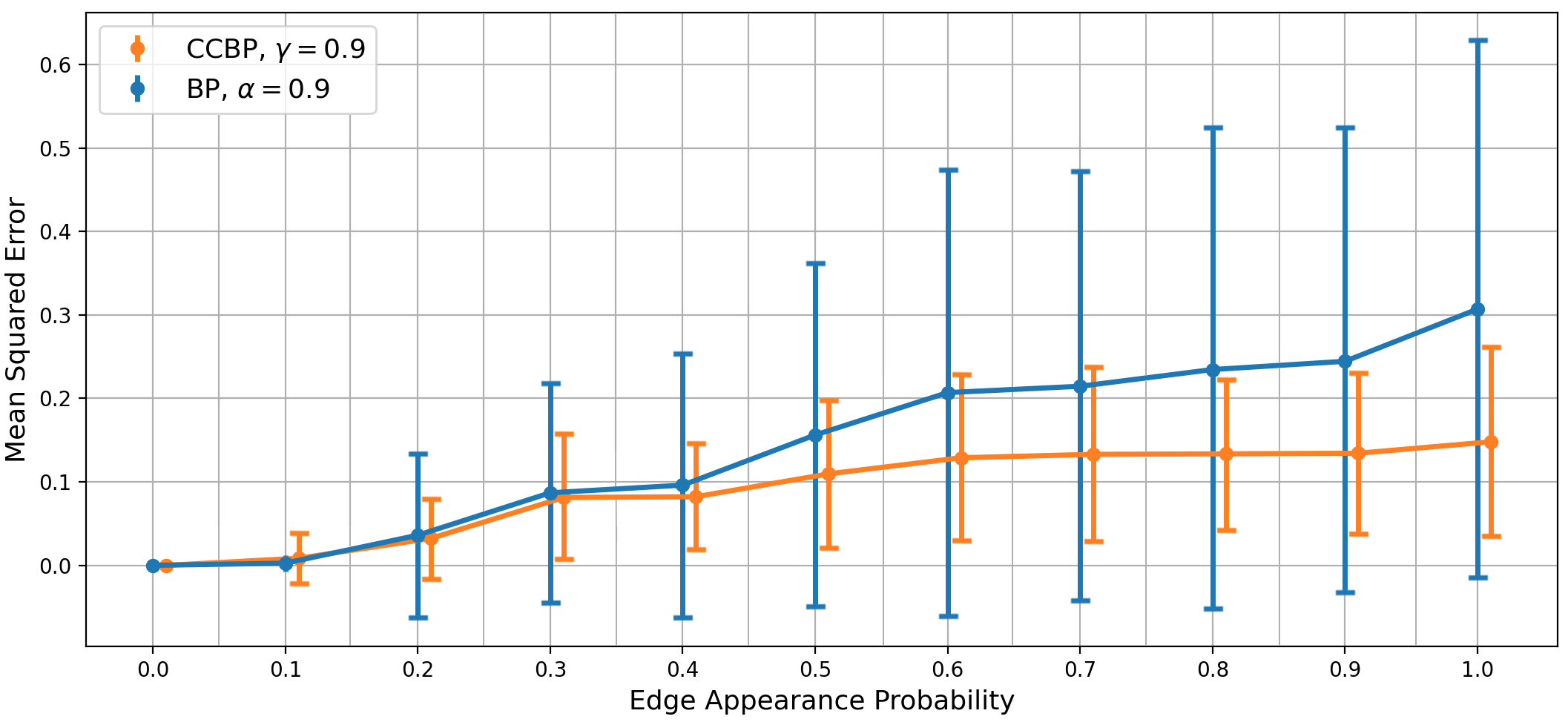}
    \caption{Accuracy of BP algorithms}
    \label{fig:edges-mse}
\end{figure}

As expected, we see that CCBP converges independent of the topology of the graph. It is also important to note that the rate of convergence is also independent of the topology. This is again a consequence of the message passing operator being a contraction
with Liptschitz constant $\gamma$.  We see that CCBP converged after at most 50 iterations in both experiments described in this section (see Figures \ref{fig:sigma-runtime} and \ref{fig:edges-runtime}).

\subsection{Image Restoration} 

In this section, we demonstrate a practical application of max-product CCBP by using this algorithm to perform image restoration. The objective of image restoration is to estimate a clean (original) image given a corrupted version of it. A classical approach is to formulate this problem as finding the optimal labelling of a graph with respect to an energy (see, e.g., \cite{blake87}, \cite{felzenszwalb06}, \cite{geman84}). The cost functions used in the energy enforce that the restored image is both piecewise smooth and consistent with the observed image. Once the image restoration model is formulated under this framework, CCBP can be used to restore the corrupted image by obtaining an approximation to the minimum of the energy.

\begin{figure}[ht]
\centering		
	\includegraphics[width=130mm]{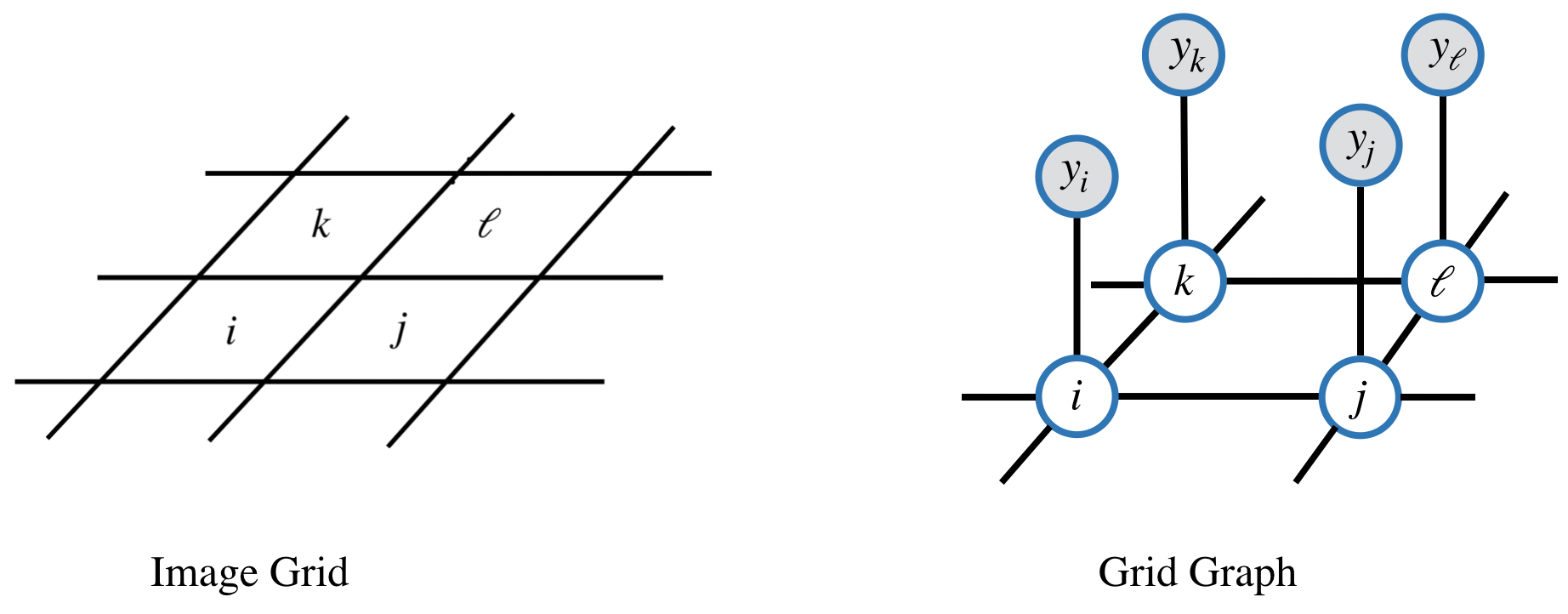}
	\caption{On the left, we see a small region of an image with pixels $i,j,k,\ell$ in the image grid. On the right, we see the corresponding nodes in the grid graph. In addition, each node has an observed pixel value which is depicted by the gray colored node. }
	\label{grid_graph}
\end{figure} 

To formalize this approach, let $\mathcal I$ be an $n\times m$ image and let $G_{n,m}=(V,E)$ be an $n\times m$ grid graph that provides a graphical representation of the image. Each pixel in the image corresponds to a node in this graph and the edges connect neighboring pixels as shown in Figure \ref{grid_graph}. Under the assumption that $\mathcal I$ is an 8-bit image, the label space $\Omega=\{0,\ldots,255\}$ is the set of possible pixel values. Let $x=(x_{1},\ldots,x_{nm})\in\Omega^{nm}$ denote a labelling of the graph and note that a labelling defines an image. Let $y=(y_1\ldots,y_{mn})\in\Omega^{nm}$ denote the observed (corrupted) image. The cost functions used in this application are
\begin{equation*}
g_i(x_i)=\vert x_i-y_i\vert^2\quad\text{and}\quad h_{ij}(x_i,x_j)=\lambda\min\big(\,\vert x_i-x_j\vert^2,\tau\big).
\end{equation*}

The unary cost $g_i$ enforces that pixels in the restored image are consistent with the observation. The pairwise cost $h_{ij}$ enforces that the restored image is piecewise smooth with spatially coherent regions.  The quadratic difference in $h_{ij}$ is bounded by $\tau\in\Omega$ to allow for large differences between neighboring pixels, which occurs when two neighbors belong to different objects. The parameter $\lambda>0$ controls how much weight is placed on the consistency versus smoothing terms in the energy, where large values of $\lambda$ result in smoother images. Under this framework, an approximation to the clean image can be obtained by minimizing the energy given by
\begin{equation*}
E(x)=\sum_{i\in V}\vert x_i-y_i\vert^2+\lambda\sum_{\{i,j\}\in E}\min\big(\,\vert x_i-x_j\vert^2,\tau\big). 
\end{equation*}

For our experiments we used the RGB image shown on the left hand side of Figure \ref{clean_corrupt}. This image is composed of three color channels, where each channel is an 8-bit image with dimensions $400\times 466$.  We restore the RGB image by separately restoring each of the three
color channels.

We generated a corrupted version of the image by adding independent noise to each pixel, by sampling from a Gaussian distribution with mean $\mu=0$ and standard deviation $\sigma=50$. In Figure \ref{clean_corrupt} we show the original image on the left and a corrupted version of this image on the right. The weights used in the operator $S$ were set uniformly and we set $\gamma=0.99$. The parameters in our image model were set as $\tau=100$ and we tried several values of $\lambda$.  We applied CCBP to restore each channel of the corrupted image and the algorithm converged after at most 8 iterations. 

\begin{figure}
\centering		
	\includegraphics[width=140mm]{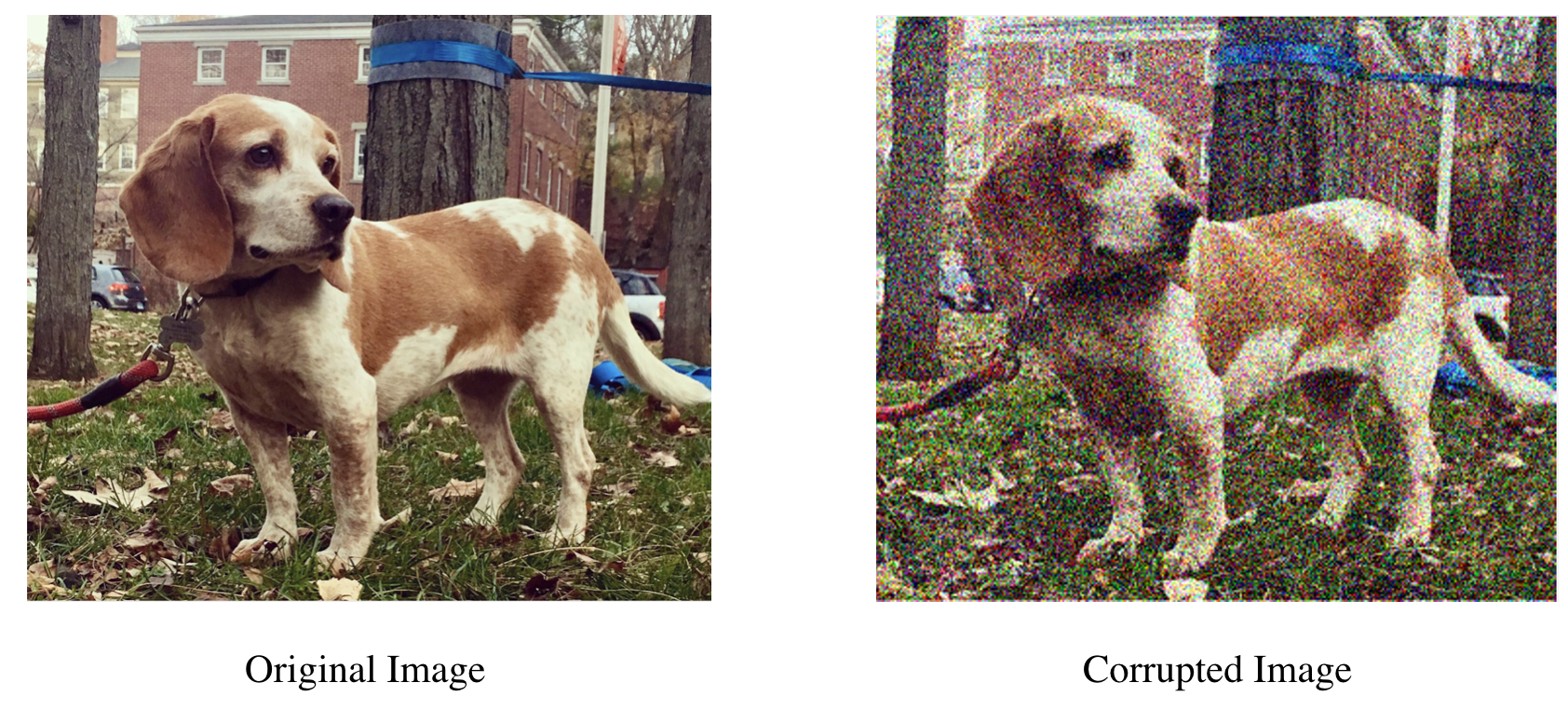}
	\caption{Data used in the image restoration experiment.}
	\label{clean_corrupt}
\end{figure} 

A crucial component of this model is choosing a good value for the parameter $\lambda$.  We obtained the best results with $\lambda=3$ as shown in Figure \ref{beta3}. The restored image is smooth almost everywhere.  Importantly the result also preserves sharp discontinuities at the boundary of different objects .

\begin{figure}
\centering		
	\includegraphics[width=140mm]{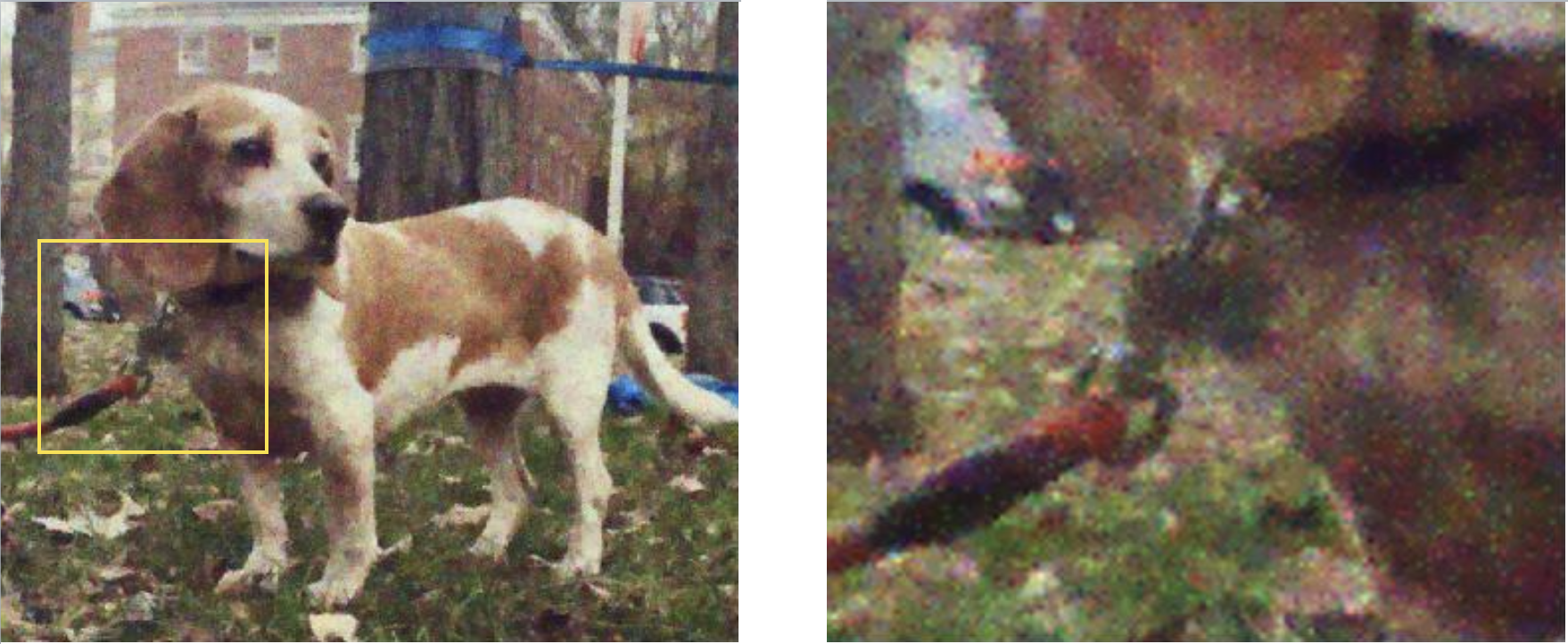}
	\caption{Image restoration using CCBP with $\lambda=4$.}
	\label{beta3}
\end{figure}


\section{Discussion}\label{sec:discussion}

The purpose of this paper is to introduce an alternative to traditional belief propagation that we refer to as convex combination belief propagation. The primary advantage of this new method is that it converges to unique fixed points on graphs with arbitrary topology, independent of the initialization of messages. We provided a characterization of the beliefs obtained from the min-sum version of the algorithm in the case when the graph is acyclic. In addition, we demonstrated the practical application of this algorithm for image restoration.

Although CCBP has good theoretical properties, one drawback of the method is that the resulting beliefs may not provide a good approximation of the exact min-marginals (or marginal distributions) in the case of the min-sum (or sum-product) algorithm. For example, traditional BP is guaranteed to compute the exact min-marginals (or marginals) when the graph is acyclic. However, we showed the beliefs obtained with CCBP are the min-marginals of a weighted energy when the graph is acyclic. Although the beliefs may differ from the exact min-marginals of the unweighted energy, they still provide useful information. To illustrate this point, we provide an example of using CCBP to obtain a fixed point in the spin glass model from statistical physics.

\begin{exmp}\label{ex:accuracy_ms}
Let $G=(V,E)$ be a complete graph with twelve nodes and let $\Omega =\{-1,1\}$ be a set of possible outcomes for each random variable. Consider the energy given by
\begin{equation*}
E(x)=\sum_{i\in V}y_i x_i +\sum_{\{i,j\}\in E}\lambda_{ij}x_i x_j
\end{equation*} 
Each $y_i$ was independently sampled from a unifor distribution over $\Omega$ and each $\lambda_{ij}$ was independently sampled from a standard normal distribution. 

Let $\gamma=0.9$ and $w_{ki} = 1/(d(i)-1)$ in CCBP. Let $\alpha=0.9$ be the damping factor in traditional BP.  We use the initialization $\mu^{(0)}_{ij}=(1,1)$.  For each algorithm, we computed the resulting beliefs along with the exact min-marginals of the energy. In Figure \ref{fig:accuracy_ms}, we show the value of these functions when $x_i=1$. (Note: each function was shifted to have mean zero.)
\end{exmp}

\begin{figure}[ht]
\centering		
	\includegraphics[width=150mm]{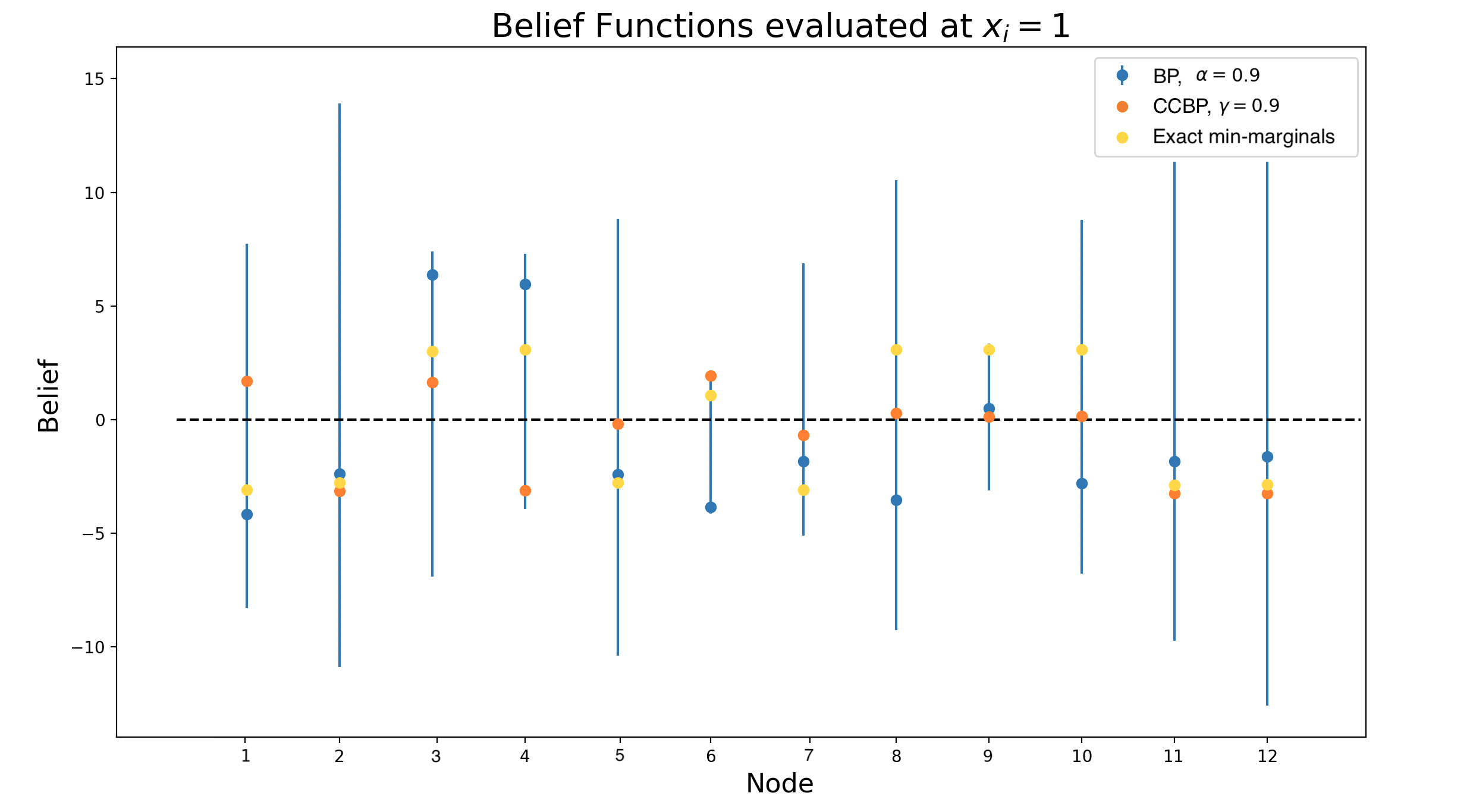}
	\caption{Beliefs obtained from min-sum algorithms and exact min-marginals of the energy from Example \ref{ex:accuracy_ms}. BP failed to converge and the vertical lines show the range of oscillation in the last 100 iterations before stopping the algorithm after 1000 iterations. CCBP converged as expected, but the resulting beliefs differ from the exact min-marginals.}
	\label{fig:accuracy_ms}
\end{figure} 

As seen in Figure~\ref{fig:accuracy_ms} the beliefs obtained with traditional BP oscillate despite the damping factor being quite large. In contrast, we see that the beliefs obtained with CCBP converge. We have seen CCBP obtains good approximations to min-marginals across many problem instances, but there is no theoretical bound on the quality of this approximation. This is an issue that is shared by both convex combination and damped belief propagation (when damped BP converges). 

In damped BP, the message passing operator can have multiple fixed points. Some fixed points may be attractive while others may be repulsive which makes the dynamics of fixed point iteration unstable. CCBP is always guaranteed to converge. Although the resulting beliefs may differ from the exact min-marginals, the beliefs provide useful information to obtain a good configuration. 

Consider the problem of obtaining an optimal (minimum energy) configuration in Example \ref{ex:accuracy_ms}. The beliefs and min-marginals in Figure \ref{fig:accuracy_ms} were centered to have mean zero. Since the state space consists of exactly two elements each belief is positive for one label and negative for the other. In this framework, the value 0 functions as a decision boundary where favorable states correspond to a negative belief.

The beliefs obtained with damped BP in Example \ref{ex:accuracy_ms} are problematic because they oscillate. In some special cases of oscillations, the beliefs can be useful as long as they choose the same state for the entire period of the oscillation. However, the beliefs in Example \ref{ex:accuracy_ms} do not fall into this category. Instead the beliefs are uninformative because they oscillate about the decision boundary. Given these circumstances, it is impossible to confidently choose a labelling with these beliefs. 

Although the beliefs obtained with CCBP differ from the exact min-marginals, they agree on the optimal state for all but one node (see Figure \ref{fig:accuracy_ms}). This highlights a strength of this algorithm, namely that it converges on difficult inference tasks and the beliefs tend to align with the exact min-marginals. Thus, one conclusion to draw from this example is that CCBP provides useful information in decision-based applications. 

We conclude by noting that CCBP is not only guaranteed to converge but also converges quickly. As discussed above the resulting beliefs provide useful information in the context of decision making, a central component of an intelligence system. Many applications require powerful algorithms that can handle complex and large scale data sets. CCBP is a natural fit for these applications because it is designed to converge on the most challenging problems. 

\section*{Acknowledgement}

We would like to thank the reviewer who provided very helpful comments that improved the overall structure and content in this work. 

\bibliographystyle{plain}
\bibliography{manuscript}

\appendix
\section{Message Passing in a Metric Space}

In this appendix we prove the space $\mathcal M$ is complete with respect to the metric $d$ from Proposition \ref{complete}. In the next lemma, we define a distance function over the set of positive reals and prove that this pair is a complete metric space. Then we extend this distance function into a metric that is defined over the product space $\mathcal M$.

\begin{lemma}\label{simple_metric}
Let $f:\mathbb R_+\times\mathbb R_+\rightarrow\mathbb R$ be the distance function given by 
\begin{equation*}
f(x_i,x_j)=\vert\log\,x_i-\log\, x_j\,\vert
\end{equation*}
for any $x_i,x_j\in\mathbb R_+$, then the pair $(\mathbb R_+,f)$ is a complete metric space.
\end{lemma}
\begin{proof}
It is clear that $f$ is non-negative, symmetric, and that $d(x_i,x_j)=0$ if and only if $x_i=x_j$. The triangle inequality holds for any $x_i,x_j,x_k\in\mathbb R_+$ by
\begin{align*}
    f(x_i,x_j)&=\vert\log\,x_i-\log\, x_j\vert\\
              &\leq\vert\log\,x_i-\log\, x_k\vert+\vert\log\,x_k-\log\, x_j\vert\\
              &=f(x_i,x_k)+f(x_k,x_j).
\end{align*}
To show that this space is complete, choose any Cauchy sequence $\{x_n\}\subset\mathbb{R}_+$ and note that this sequence can be written as $\{x_n\}=\{e^{y_n}\}$ with $y_n=\log\,x_n$. The sequence $\{y_n\}\subset\mathbb R$ must be Cauchy with respect to the Euclidean metric because for any $\epsilon>0$ there exists an $N>0$ such that $f(x_n,x_m)<\epsilon$ for all $n,m>N$, which implies that
\begin{equation*}
\vert y_n-y_m\vert=\vert\log\,x_n-\log\,x_m\vert=f(x_n,x_m)<\epsilon.
\end{equation*}
Since $\{y_n\}$ is a Cauchy sequence in a complete space, there exists some $y\in\mathbb{R}$ such that $y_n\rightarrow y$ and hence $x_n=e^{y_n} \rightarrow e^y$. 
\end{proof}

\setcounter{prop}{0}
\begin{prop}
Let $d:\mathcal M\times\mathcal M\rightarrow\mathbb R$ be the distance function given by
\begin{equation*}
d(\mu,\nu)=\max_{i\in V}\max_{j\in N(i)}\max_{x_j}\big\vert\log\,\mu_{ij}(x_j)-\log\,\nu_{ij}(x_j)\big\vert,
\end{equation*}
then the pair $\big(\mathcal M,d\big)$ is a complete metric space. 
\end{prop}
\begin{proof}
It is clear that $d$ is non-negative, symmetric, and that $d(\mu,\nu)=0$ if and only if $\mu=\nu$. To show the triangle inequality, choose any $\mu,\nu,\lambda\in\mathcal M$ and observe that
\begin{align*}
d(\mu,\nu)&=\max_{i\in V}\max_{j\in N(i)}\max_{x_j}\big\vert\log\,\mu_{ij}(x_j)-\log\,\nu_{ij}(x_j)\big\vert\\
&\leq\max_{i\in V}\max_{j\in N(i)}\max_{x_j}\Big(\,\big\vert\log\,\mu_{ij}(x_j)-\log\,\lambda_{ij}(x_j)\vert+\vert\log\,\lambda_{ij}(x_j)-\log\,\nu_{ij}(x_j)\big\vert\,\Big)\\
&\leq\max_{i\in V}\max_{j\in N(i)}\max_{x_j}\big\vert\log\,\mu_{ij}(x_j)-\log\,\lambda_{ij}(x_j)\vert+\max_{i\in V}\max_{j\in N(i)}\max_{x_j}\vert\log\,\lambda_{ij}(x_j)-\log\,\nu_{ij}(x_j)\big\vert\\
&=d(\mu,\lambda)+d(\lambda,\nu).
\end{align*}
Now choose any Cauchy sequence $\{\mu^{(n)}\}\subset\mathcal M$, then $\{\mu_{ij}^{(n)}(x_j)\}\subset\mathbb{R}_+$ is a Cauchy sequence in $(\mathbb R_+,f)$ because for any $\epsilon>0$ there exists an $N>0$ such that for all $n,m>N$
\begin{align*}
f\big(\mu_{ij}^{(n)}(x_j),\,\mu_{ij}^{(m)}(x_j)\big)&=\big\vert \log\,\mu_{ij}^{(n)}(x_j)-\log\,\mu_{ij}^{(m)}(x_j)\big\vert \\ 
&\leq\max_{i\in V}\max_{j\in N(i)}\max_{x_j}\,\big\vert\log\,\mu_{ij}^{(n)}(x_j)-\log\,\mu_{ij}^{(m)}(x_j)\big\vert\\
&=d\big(\mu^{(n)},\mu^{(m)}\big)\\
&<\epsilon
\end{align*}
Given that the pair $(\mathbb{R}_+,f)$ is a complete metric space by Lemma \ref{simple_metric}, there exists some $\mu$ such that $\mu^{(n)}_{ij}(x_j)\rightarrow \mu_{ij}(x_j)$.  Thus, the space $(\mathcal M,d)$ is complete because $\mu^{(n)}\rightarrow \mu\in\mathcal M$ by 
\begin{equation*}
    d(\mu^{(n)},\mu)=\max_{i\in V}\max_{j\in N(i)}\max_{x_j} f\big(\mu^{(n)}_{ij}(x_j),\mu_{ij}(x_j)\big)\rightarrow0.
\end{equation*}
\end{proof}

\section{Sum-Product Algorithm}

In this appendix, we present the the sum-product version of CCBP. The operator in the sum-product algorithm is analogous to the one derived from the max-product equations.

\begin{defn}
The operator $S:\mathcal{M}\rightarrow\mathcal{M}$ in the CCBP sum-product algorithm is
\begin{equation*}
\big(S\mu\big)_{ij}(x_j)=\sum_{x_i}\phi_i(x_i)\psi_{ij}(x_i,x_j)\Big(\prod_{k\in N(i)\backslash j}\mu_{ki}(x_i)^{w_{ki}}\Big)^\gamma ,
\end{equation*}
where the weights must be non-negative with $\sum\limits_{k\in N(i)\backslash j}w_{ki}\leq1$ and $\gamma\in(0,1)$ 
\end{defn}

Similar to the case of the max-product operator we show the sum-product CCBP operator is contractive.

\begin{lemma}\label{sp_contraction}
The operator $S$ is contractive with Lipschitz constant $\gamma$.
\end{lemma}
\begin{proof}
Choose any $\mu,\nu\in\mathcal M$, then 
\begin{align*}
\big(Sm\big)_{ij}(x_j)&=\sum_{x_i} \phi_i(x_i)\psi_{ij}(x_i,x_j)\Big(\prod_{k\in N(i)\backslash j}\mu_{ki}(x_i)^{w_{ki}}\Big)^\gamma \\
&=\sum_{x_i} \phi_i(x_i)\psi_{ij}(x_i,x_j)\prod_{k\in N(i)\backslash j}\nu_{ki}(x_i)^{\gamma w_{ki}}\prod_{k\in N(i)\backslash j}\frac{\mu_{ki}(x_i)^{\gamma w_{ki}}}{ \nu_{ki}(x_i)^{\gamma w_{ki}}} \\
&\leq\bigg(\sum_{x_i} \phi_i(x_i)\psi_{ij}(x_i,x_j)\prod_{k\in N(i)\backslash j}\nu_{ki}(x_i)^{\gamma w_{ki}} \bigg)\bigg(\max_{x_i}\;\prod_{k\in N(i)\backslash j}\,\frac{\mu_{ki}(x_i)^{\gamma w_{ki}}}{ \nu_{ki}(x_i)^{\gamma w_{ki}}}\,\bigg)\\
&=\big(S\nu\big)_{ij}(x_j)\;\max_{x_i}\,\prod_{k\in N(i)\backslash j}\frac{\mu_{ki}(x_i)^{\gamma w_{ki}}}{ \nu_{ki}(x_i)^{\gamma w_{ki}}}.
\end{align*}
Taking the logarithm of both sides yields that
\begin{align*}
\log\big(S\mu\big)_{ij}(x_j)&\leq\log\bigg(\big(S\nu\big)_{ij}(x_j)\;\max_{x_i}\prod_{k\in N(i)\backslash j}\frac{\mu_{ki}(x_i)^{\gamma w_{ki}}}{ \nu_{ki}(x_i)^{\gamma w_{ki}}}\bigg)\\
&= \log\big(S\nu\big)_{ij}(x_j)+\gamma\,\max_{x_i}\sum_{k\in N(i)\backslash j}w_{ki}\log\,\frac{\mu_{ki}(x_i)}{ \nu_{ki}(x_i)} \\
&\leq\log\big(S\nu\big)_{ij}(x_j)+\gamma\,\max_{x_i}\sum_{k\in N(i)\backslash j}w_{ki}\,\Big\vert\log\,\frac{\mu_{ki}(x_i)}{\nu_{ki}(x_i)} \,\Big\vert
\end{align*}
\begin{equation*}
\implies\log\big(S\mu\big)_{ij}(x_j)-\log\big(S\nu\big)_{ij}(x_j)\leq\gamma\,\max_{x_i}\sum_{k\in N(i)\backslash j}w_{ki}\,\Big\vert\log\,\frac{\mu_{ki}(x_i)}{\nu_{ki}(x_i)}\Big\vert. 
\end{equation*}
Since this inequality holds when $\mu$ and $\nu$ are interchanged, we can take the absolute value of the left hand side. Moreover, given that the above inequality holds for any $x_j\in\Omega$, it must hold for the maximum over $x_j$ the left hand side. 
\begin{align*}
\max_{x_j}\Big\vert\log\,\frac{ (S\mu)_{ij}(x_j)}{(S\nu)_{ij}(x_j)}\,\Big\vert
&\leq\gamma\,\max_{x_i}\sum_{k\in N(i)\backslash j}w_{ki}\,\Big\vert\log\,\frac{\mu_{ki}(x_i)}{\nu_{ki}(x_i)}\,\Big\vert\\ 
&\leq\gamma\,\max_{i\in V}\max_{j\in N(i)}\max_{x_i}\sum_{k\in N(i)\backslash j}w_{ki}\,\Big\vert\log\,\frac{\mu_{ki}(x_i)}{\nu_{ki}(x_i)}\,\Big\vert\\ 
&\leq\gamma\,\max_{i\in V}\max_{j\in N(i)}\max_{k\in N(i)\backslash j}\max_{x_i}\,\Big\vert\log\,\frac{\mu_{ki}(x_i)}{\nu_{ki}(x_i)}\,\Big\vert\\ 
&=\gamma\,\max_{i\in V}\max_{j\in N(i)}\max_{x_i}\,\Big\vert\log\,\frac{\mu_{ij}(x_i)}{\nu_{ij}(x_i)}\,\Big\vert.
\end{align*}
Since the edge $\{i,j\}\in E$ was chosen arbitrarily from the beginning, the inequality holds for any  $\{i,j\}\in E$.  The final result is obtained by taking the maximum of the left hand side over all the edges in the graph.
\begin{equation*}
\max_{i\in V}\max_{j\in N(i)}\max_{x_i}\Big\vert\log\,\frac{ (S\mu)_{ij}(x_j)}{(S\nu)_{ij}(x_j)}\,\Big\vert \leq\gamma\,\max_{i\in V}\max_{j\in N(i)}\max_{x_i}\,\Big\vert\log\,\frac{\mu_{ij}(x_i)}{\nu_{ij}(x_i)}\,\Big\vert
\end{equation*}
\begin{equation*}
    \implies d\big(S\mu, S\nu\big)\leq\gamma\, d(\mu,\nu)
\end{equation*}
\end{proof}  

\begin{thm}
The operator $S$ has a unique fixed point $\mu^\star\in\mathcal{M}$ and the sequence defined by $\mu^{(n+1)}:= S\mu^{(n)}$ converges to $\mu^\star$. Furthermore, after $n$ iterations
\begin{equation*}
d(S^{(n)}\mu^{(0)},\mu^\star)\leq\gamma^n\, d(\mu^{(0)},\mu^\star).
\end{equation*}
\end{thm}
\begin{proof}
The pair $(\mathcal M,d)$ is a complete metric space by Proposition \ref{complete} and $S$ is a contraction by Lemma \ref{sp_contraction}.  The result holds by applying Banach's fixed point theorem. 
\end{proof}

\end{document}